\documentclass{article}

\usepackage{arxiv}

\usepackage[utf8]{inputenc} % allow utf-8 input
\usepackage[T1]{fontenc}    % use 8-bit T1 fonts
\usepackage{hyperref}       % hyperlinks
\usepackage{url}            % simple URL typesetting
\usepackage{booktabs}       % professional-quality tables
\usepackage{amsfonts}       % blackboard math symbols
\usepackage{nicefrac}       % compact symbols for 1/2, etc.
\usepackage{microtype}      % microtypography
\usepackage{amsmath}
\usepackage{mathtools}
\usepackage{amsthm}
\usepackage{multirow}
\usepackage{graphicx}
\usepackage{algorithm2e}

\usepackage[%
        backend=bibtex,
        style=ieee,
        natbib=true,
        backref=false,
        backrefstyle=all+,
        hyperref=true,
    ]{biblatex}
    
\newtheorem{theorem}{Theorem}
\newtheorem{lemma}{Lemma}

\title{E2E-FS: An End-to-End Feature Selection Method for Neural Networks}

\author{
  Brais Cancela \\
  CITIC Research Center\\
  Universidade da Coru\~na\\
  A Coru\~na, Spain, 15008 \\
  \texttt{brais.cancela@udc.es} \\
  \And
  Ver\'onica Bol\'on-Canedo\\
  CITIC Research Center\\
  Universidade da Coru\~na\\
  A Coru\~na, Spain, 15008 \\
  \texttt{veronica.bolon@udc.es} \\
  \And
  Amparo Alonso-Betanzos\\
  CITIC Research Center\\
  Universidade da Coru\~na\\
  A Coru\~na, Spain, 15008 \\
  \texttt{ciamparo@udc.es} \\
}

\begin{document}
\maketitle

\begin{abstract}
    Classic embedded feature selection algorithms are often divided in two large groups: tree-based algorithms and lasso variants. Both approaches are focused in different aspects: while the tree-based algorithms provide a clear explanation about which variables are being used to trigger a certain output, lasso-like approaches sacrifice a detailed explanation in favor of increasing its accuracy. In this paper, we present a novel embedded feature selection algorithm, called End-to-End Feature Selection (E2E-FS), that aims to provide both accuracy and explainability in a clever way. Despite having non-convex regularization terms, our algorithm, similar to the lasso approach, is solved with gradient descent techniques, introducing some restrictions that force the model to specifically select a maximum number of features that are going to be used subsequently by the classifier. Although these are hard restrictions, the experimental results obtained show that this algorithm can be used with any learning model that is trained using a gradient descent algorithm.
\end{abstract}

% keywords can be removed

\section{Introduction}
\label{sec:introduction}

High dimensional problems are very common nowadays and pose an important challenge for Machine Learning researchers. Dealing with thousands or even millions of features is not practical, particularly because some of them are redundant or not informative.  It is therefore important to correctly identify the relevant features for a given task, being this process known as \textit{feature selection} (FS). Reducing the dimensionality of a problem has several acknowledged advantages such as improving interpretability (and therefore explainability), reducing execution times and, in some cases, improving learning performance \cite{guyon2008feature}.

Feature selection methods can be grouped into classifier-dependent approaches (wrappers and embedded methods) and classifier-independent (filters). On the one hand, filters use independent metrics (such as mutual information, correlation, or statistics) to decide which features are more relevant with respect to the predictive class. In this way, the selected features are generic, and the process to extract them is usually not computationally expensive. Examples of filter approaches are Mutual Information (MI) \cite{ross2014mutual}, ReliefF \cite{kononenko1997overcoming} or the Infinite Feature Selection variants, InfFS \cite{roffo2015infinite} and ILFS \cite{roffo2017infinite}.

On the other hand, wrappers and embedded procedures use the performance of a learning method (e.g. a classifier) to determine the subset of relevant features. Wrappers search through the space of features using the accuracy of a particular classifier to determine the usefulness of a candidate feature subset. This approach tends to be computationally expensive, and the selected features are specific for the classifier used to obtain them. Embedded methods are halfway wrappers and filters (in terms of their computational expense) and determine the relevant features through the training process of a classifier. Embedded methods are less computationally expensive and less prone to overfitting than wrappers, and have the additional advantage that both feature selection and  classification training can be made at the same time. 

One of the most well-known embedded methods is Recursive Feature Elimination for Support Vector Machine (SVM-RFE) \cite{guyon2002gene}, which computes the importance of the features in the process of training a SVM. More recently, the Saliency-based Feature Selection (SFS) method \cite{cancela2020scalable} aims to use the Saliency technique \cite{simonyan2013deep} to infer the most relevant features. However, both methods have a high computational cost, as they require to train a classifier several times to obtain a good result. Lasso \cite{tibshirani1996regression} is also very popular, as it is based on the extracted subset that included shape and density features. However, contrary to the filter methods, there is no control about the number of features that are finally chosen.

In this paper we aim to merge the best characteristics of both filter methods and the Lasso approach into one unique algorithm, that we called End-to-End Feature Selection (E2E-FS). The advantages of this approach are:
\begin{enumerate}
    \item It can be used with any model that is trained by using gradient descent techniques. Thus, it is not restricted to classification problems.
    \item Similar to Lasso, the feature selection is performed at the same time the model is trained, considerably reducing the computational cost. We can avoid the computation of multiple models, like in SVM-RFE.
    \item Similar to ranker filter methods, and contrary to Lasso, we can specify the maximum number of features that are finally selected while still solving it with gradient descent techniques.
    \item It is very efficient in terms of both computational time and memory, as only a vector of the size of the number of initial features is required.
\end{enumerate}

To our knowledge, this is the first embedded method that can specify an exact number of features and train a learning model in just one step, without using a recursive approach. Only constraints and regularization parameters are used to obtain the final model.

The rest of the paper is organized as follows: first, we will describe the intuition behind our idea; second, we will provide an implementation of our E2E-FS algorithm; next we will report experimental results for a wide range of public datasets, and finally, we will offer some conclusions and future work.

\section{E2E-FS Algorithm}
\label{sec:E2E-FS_algorithm}

For our model, we will first present the idea behind our algorithm before providing an approach to implement it. Let $\mathbf{X} \in \mathcal{R}^{N \times F}$ be our input data, where $N$ are the number of instances and $F$ are the total number of different features. Let $\mathbf{Y} \in \mathcal{R}^{N \times C}$ be the expected output, where $C$ is the number of classes. For the sake of simplicity, only classification problems will be taken into account, although our approach can also be used in any other problem (for instance, regression) that can be solved by using gradient descent techniques, without making any other modification. Thus, let $\mathbf{\tilde{Y}} = f(\mathbf{X}; \mathbf{\Theta}) \in \mathcal{R}^{N \times C}$ be our classification model, where $\mathbf{\Theta}$ are the classifier parameters. Our aim is to solve a minimization problem by forcing the classifier to only select a maximum number of features, denoted by $M$. Formally speaking, our algorithm aims to solve the following minimization problem:
\begin{equation}
    \label{eq:formulation}
    \begin{aligned}
    & \underset{\mathbf{\Theta}, \mathbf{\gamma}}{\text{minimize}}
    & & \mathcal{L}(f(\mathbf{\gamma} \circ \mathbf{X}; \mathbf{\Theta}), \mathbf{Y}) \\
    & \text{subject to}
    & & \mathbf{\gamma} \in \{0, 1\}^F, \\
    & 
    & & \|\mathbf{\gamma}| \|_1 \leq M.
    % & & \sum_{i=1}^{|\mathbf{\gamma}|} \gamma_i  \leq M.
    \end{aligned}
\end{equation}
where $\mathcal{L}$ is the loss function, $M$ is the maximum number of features we wish to use; and $\mathbf{\gamma}$ is the mask layer. The idea is simple: we aim to train a classification problem while introducing a binary mask layer, which will be in charge of selecting the most relevant $M$ features (or less). Initially, this is not a problem that can be solved by using gradient descent techniques, as the binary mask is not differentiable. However, some approximations can be made to fulfill the requirements.

\section{E2E-FS Implementation}
\label{sec:E2E-FS_implementation}

The first decision is to select the shape of $\mathbf{\gamma}$. The easiest approach is to select $\mathbf{\gamma} \in \mathcal{R}^{M \times F}$, and then forcing $\| \mathbf{\gamma} \|_{\infty} = 1$ and $\| \mathbf{\gamma} \|_1 \leq 1$, that is, all zeros but one in each row, and one non-zero per column, at most. This can be solved by using gradient descent techniques and the $l_{1-2}$ regularization \cite{yin2015minimization} in the same way described in \cite{lyu2019autoshufflenet} for learning permutation matrices. However, this approach will require a huge amount of memory space in big data environments. For instance, selecting $10000$ features from an initial dataset with more than $100000$ variables will require, when using 32 floating-point precision, near 4GB only to store the $\mathbf{\gamma}$ matrix.

For that reason, we decided to develop a different solution that only requires a vector of size $F$, that is, $\mathbf{\gamma} \in \mathcal{R}^{F}$. This solution will be solved by only introducing regularization parameters to the loss function.

\subsection{E2E-FS using soft regularization techniques}

We fulfill the restrictions exposed in Eq. \ref{eq:formulation} by only using regularization terms. Thus, this solution, called \textit{E2E-FS}, transforms the initial problem into
\begin{equation}
    \label{eq:E2E-FS_soft}
    \begin{aligned}
    & \underset{\mathbf{\Theta}, \tilde{\mathbf{\gamma}}}{\text{minimize}}
    & & \overbrace{\mathcal{L}(f(\tilde{\mathbf{\gamma}} \circ \mathbf{X}; \mathbf{\Theta}), \mathbf{Y})}^{\mathcal{L}_f} \\
    & \text{subject to}
    & & \tilde{\mathbf{\gamma}} \in [0, 1]^F, \\
    & 
    & & \overbrace{\underbrace{\| \tilde{\mathbf{\gamma}} \|_1 - \| \tilde{\mathbf{\gamma}} \|^2_2 }_{\mathcal{L}_{1-2}} ~+~
     \underbrace{(1 + \mu)~\max(0, | M - \| \tilde{\mathbf{\gamma}} \|_1 |)}_{\mathcal{L}_{M}}}^{\mathcal{L}_{\tilde{\mathbf{\gamma}}}} = 0,\\
    % & & \sum_{i=1}^{|\mathbf{\gamma}|} \gamma_i  \leq M.
    \end{aligned}
\end{equation}
where $\mu > 0$ is a hyper-parameter (set by default to 1).
% \begin{equation}
%     \label{eq:E2E-FS_soft}
%     \underset{\mathbf{\Theta}, \mathbf{\gamma}}{\text{minimize}}
%     \quad \overbrace{\mathcal{L}(f(\tilde{\mathbf{\gamma}} \mathbf{X}; \mathbf{\Theta}), \mathbf{Y})}^{\mathcal{L}_f} + \overbrace{\| \tilde{\mathbf{\gamma}} \|_1 - \| \tilde{\mathbf{\gamma}} \|^2_2 }^{\mathcal{L}_{1-2}} ~+~
%      \underbrace{\max(0, | M - \| \tilde{\mathbf{\gamma}} \|_1 |)}_{\mathcal{L}_{M}}
%      % ~+  \underbrace{\sum_{i=1}^F \max(0, 1 - \gamma_i) }_{\mathcal{L}_{1}}
% \end{equation}

The intuition behind the idea is simple: $\mathcal{L}_{1-2}$ is used to force $\tilde{\mathbf{\gamma}}$ values to be binary (either $0$ or $1$), while $\mathcal{L}_{M}$ ensures the summation of all values in $\tilde{\mathbf{\gamma}}$ are near to the desired maximum number of features $M$. %; finally, $\mathcal{L}_{1}$ is used to prevent $\mathbf{\gamma}$ to have higher values than $1$. The latter is optional, but it allows us to have more control over the $\mathbf{\gamma}$ mask, as the gradient of $\mathcal{L}_{f}$ w.r.t. $\mathbf{\gamma}$ vanishes whenever $\mathbf{\gamma} > 1$. % We also include an $l2$ regularization over $\mathbf{\gamma}$ in $\mathcal{L}_f$ to prevent $\mathbf{\gamma}$ to have high values during the first iterations. %whenever the regularization terms are not contributing too high. 
By default, we initialized $\tilde{\mathbf{\gamma}} = 1$. % Initializing it with values too close to $0$ can make some features to be prematurely discarded.

\vspace{0.3cm}\noindent\textbf{Complexity}:\quad The complexity of this approach remains at $\mathcal{O}(\mathcal{O}_f)$, as this approach only introduces regularization terms over the model $f$.

% \subsubsection{Implementation details}
\vspace{0.3cm}\noindent\textbf{Implementation details}:\quad As the restrictions in this approach are more relaxed, we cannot guarantee that $M$ features are selected after a fixed number of epochs. Instead, we need to check the loss function to confirm it. We can assure that
\begin{equation}
    \mathcal{L}_{\tilde{\mathbf{\gamma}}} = \mathcal{L}_{1-2} + \mathcal{L}_{M} = 0 ~~ \Rightarrow ~~ nnz(\tilde{\mathbf{\gamma}}) = M
\end{equation}

Even if we are changing Eq. \ref{eq:formulation} into a problem with regularization terms, we still need to satisfy the constraints to obtain the desired result. Thus, we need to force $\mathcal{L}_{1-2}$ and  $\mathcal{L}_{M}$ to be zero. To do so, we define the gradient w.r.t. $\tilde{\mathbf{\gamma}}$ as
\begin{equation}
    \label{eq:soft_gradient}
    \dfrac{\partial \mathcal{L}_{f, \tilde{\mathbf{\gamma}}}}{\partial \tilde{\mathbf{\gamma}}} = (1 - \alpha) \dfrac{\partial \mathcal{L}_{f}}{\partial \tilde{\mathbf{\gamma}}} + \alpha \dfrac{\partial \mathcal{L}_{\tilde{\mathbf{\gamma}}}}{\partial \tilde{\mathbf{\gamma}}},
\end{equation}
where $\alpha \in [0, 1]$ is a hyper-parameter. This parameter will control the gradient focus between the classification loss and the regularization terms. Note that the same $\alpha$ hyper-parameter controls both $\mathcal{L}_{1-2}$ and $\mathcal{L}_{M}$ terms. In the proof of convergence we will show why these two terms should always be treated as a whole.

\begin{figure*}
	\centering
	\begin{minipage}[]{0.45\linewidth}
	    \centering
 		\includegraphics[width=0.99\textwidth]{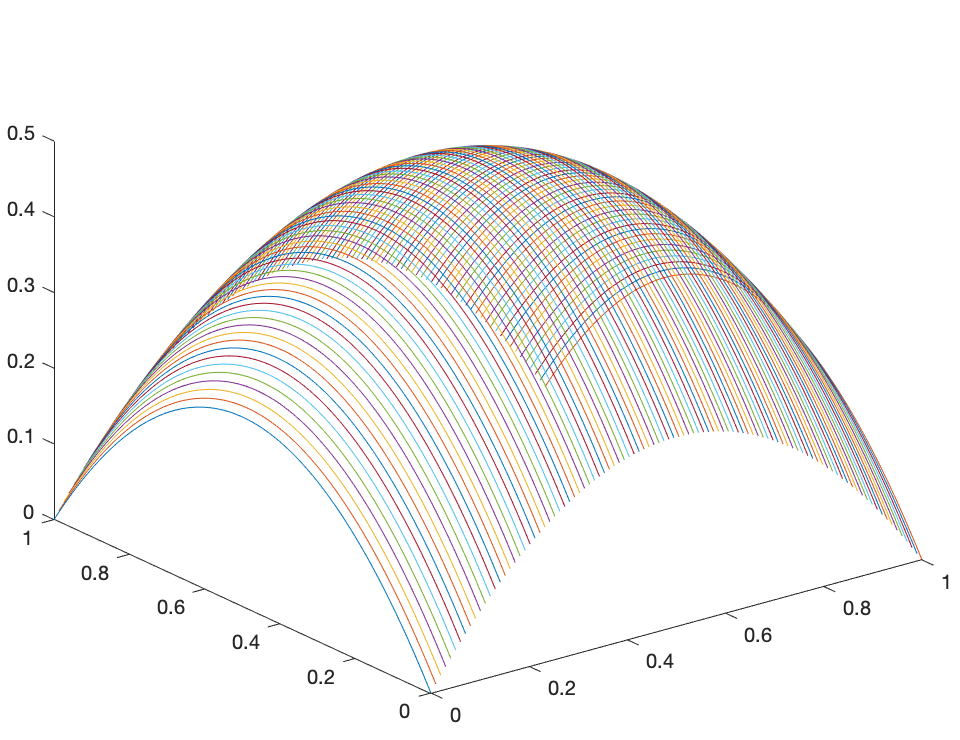}\\
 		(a) $\mathcal{L}_{1-2}$
	\end{minipage}
	\begin{minipage}[]{0.45\linewidth}
	    \centering
 		\includegraphics[width=0.99\textwidth]{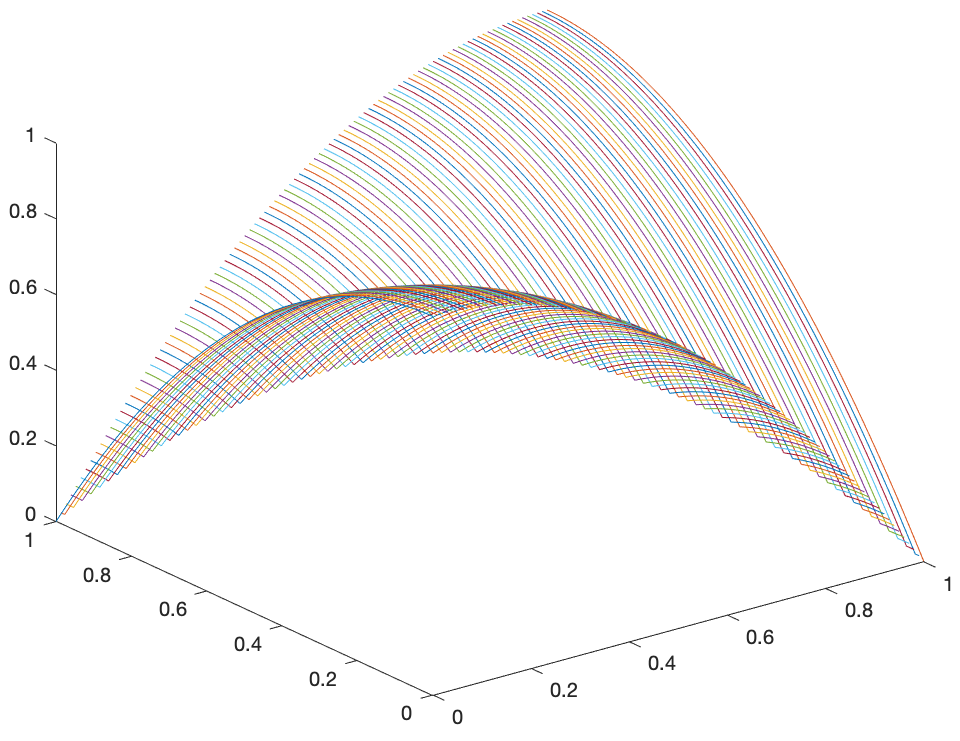}\\
 		(b) $\mathcal{L}_{1-2} + \mathcal{L}_{M}$
	\end{minipage}
\caption{
 $\mathcal{L}_{1-2}$ loss for two variables and $M = 1$. Although is a concave function (a) with minimum values in undesirable states (all zeros and all ones), it is possible to remove them by adding $\mathcal{L}_{M}$ to the loss equation.
}
\label{fig:non_convex}
\end{figure*}

When applying gradient descent to minimize $\mathcal{L}_{f, \mathbf{\gamma}}$, this approach presents a couple of issues. Below we will mention them, providing also some alternatives to solve them.
\begin{enumerate}
    % \item \textit{Some important features can quickly disappear and, they will never appear again}. We can partially prevent this problem by using a dropout technique \cite{srivastava2014dropout} variant, randomly setting some $\mathbf{\gamma}$ values to $-\infty$. Contrary to classic dropout techniques, mean correction is not performed. By default, we set the dropout probability to a low value ($p = 0.1$).
    \item \textit{$\mathcal{L}_{1-2}$ is non-convex.} As depicted in Fig. \ref{fig:non_convex}-(a), $\mathcal{L}_{1-2}$ is a concave function that, when restricted to $\mathbf{\gamma} \in [0, 1]$, has its minimum values over the corners (that is, when $\mathbf{\gamma} = \{0, 1\}$. This is a problem because this function, when reaches its minimum, can led to some extreme configurations, like selecting all features or discarding all of them. However, when the $\mathcal{L}_{M}$ loss is also introduced (see Fig. \ref{fig:non_convex}-(b)), we can guarantee that the minimum is reached it and only if the number of selected features is equal to $M$.
    
    Besides that, we still have a concave function. This led to the gradient accelerating when going near to the minimum. But this is not a huge problem, as the purpose of this algorithm is to force every value on the vector $\mathbf{\gamma}$ to reach one of its extreme values (0 or 1). In the proof of convergence, we will also see how our combination $\mathcal{L}_{1-2} + \mathcal{L}_{M}$ can prevent our algorithm to discard more features than $M$, but also increase the separation between relevant and irrelevant features.
    
    \item \textit{The derivatives of the classification problem and the regularization terms can be of different scale}. Initially, we would like the gradient of gamma to be guided by the classification loss $\mathcal{L}_{f}$, and gradually changing the $\alpha$ parameter to $1$, forcing the training to remove the least important features. In order to ensure both values have the same scale we change Eq. \ref{eq:soft_gradient} to
    \begin{equation}
        \label{eq:soft_gradient_2}
        \dfrac{\partial \mathcal{L}_{f, \tilde{\mathbf{\gamma}}}}{\partial \tilde{\mathbf{\gamma}}} = \beta \left( (1 - \alpha) z \left(\dfrac{\partial \mathcal{L}_{f}}{\partial \tilde{\mathbf{\gamma}}}\right) + \alpha z\left(\dfrac{\partial \mathcal{L}_{\tilde{\mathbf{\gamma}}}}{\partial \tilde{\mathbf{\gamma}}} \right)\right),
        %\dfrac{\partial \mathcal{L}_{f, \mathbf{\gamma}}}{\partial \mathbf{\gamma}} =  \beta \left( (1 - \alpha) \dfrac{\partial \mathcal{L}_{f}}{\partial \mathbf{\gamma}} + \alpha~z_f\left(\dfrac{\partial \mathcal{L}_{1-2}}{\partial \mathbf{\gamma}} + \dfrac{\partial \mathcal{L}_{M}}{\partial \mathbf{\gamma}} \right) \right),
    \end{equation}
    where 
    \begin{align}
	    z(x) = \frac{x}{\|x\|},
    \end{align}
    is the normalization equation. We set a value in $\beta$ to prevent a perfect classification to vanish the gradient. By default, $\beta = 1$.
    
    Again, it is difficult to establish a good $\alpha$ value, as a low value can led the restrictions not to be fulfilled, whereas a high value can create the binary matrix $\mathbf{\tilde{\gamma}}$ without taking the information of the problem (encoded in $\mathcal{L}_{f}$) into account. To prevent this issue, we decided to substitute $\alpha$ by a moving factor defined as
    \begin{equation}
        \alpha^t_T = \min(1, t/T),
    \end{equation}
    where $t$ is the training iteration and $T$ is an hyper-parameters that controls how smooth do we want to introduce the $\mathcal{L}_{\tilde{\mathbf{\gamma}}}$ loss in the $\mathbf{\tilde{\gamma}}$ update. The moving $\alpha^t_T$ allows to gradually change the gradient importance from the classification to the regularization loss. 
    
    \item \textit{It is difficult to remove new features as $nnz(\mathbf{\gamma})$ approaches to $M$}. This problem is related to the $\mathcal{L}_{1-2}$ behavior, as the gradient descent of values higher than $0.5$ goes to the direction of $1$ rather than to $0$. In a similar way as exposed earlier, we prevent this by introducing a variant of our E2E-FS algotithm, called \emph{E2E-FS-Soft}, that substitutes the parameter $M$ by a moving value
    \begin{align}
	    M_\rho & = 
        \begin{cases}
            ~ (1 - \rho) M & \text{if }  nnz(\mathbf{\gamma}) > M\\
            ~ nnz(\mathbf{\gamma}) & \text{otherwise}
        \end{cases}
    \end{align}
    By default, we set $\rho = 0.75$. The $M_\rho$ parameter allows us to remove features faster. Note that this moving factor can cause the algorithm to remove more features than expected.
\end{enumerate}

% \subsubsection{Proof of Convergence}
\vspace{0.3cm} \noindent\textbf{Proof of Convergence}:\quad If the first place, we want to explain why the configuration $\mathcal{L}_{1-2} + \mathcal{L}_{M}$ is chosen, and how it can be used to successfully obtain the desired binary vector $\tilde{\mathbf{\gamma}}$.

\begin{lemma}
    \label{lemma:1_1}
    Given any $M > 0, \quad | \tilde{\mathbf{\gamma}} |_1 > M \quad \Rightarrow \quad \frac{\partial \mathcal{L}_{1-2} + \mathcal{L}_{M}}{\partial \tilde{\mathbf{\gamma}}} > 0.$
\end{lemma}
\begin{proof}
    Knowing that $| \tilde{\mathbf{\gamma}} |_1 > M$, we can rewrite $\mathcal{L}_{M} = (1 + \mu) ~ (| \tilde{\mathbf{\gamma}} |_1 - M)$. Thus, $\mathcal{L}_{1-2} + \mathcal{L}_{M} = (2 + \mu)~| \tilde{\mathbf{\gamma}} |_1 - \| \tilde{\mathbf{\gamma}} \|^2_2 - (1 + \mu)~M$. Its gradient is defined as $$\frac{\partial \mathcal{L}_{1-2} + \mathcal{L}_{M}}{\partial \tilde{\gamma_i}} = 2~(1 - \tilde{\gamma_i}) + \mu, \quad \forall i \in (1 \ldots F).$$ As, by definition, $\tilde{\gamma_i} \in [0, 1]$, we have that $$\min\left( \frac{\partial \mathcal{L}_{1-2} + \mathcal{L}_{M}}{\partial \tilde{\gamma_i}} \right) = \mu, \quad \forall i \in (1 \ldots F).$$ Having $\mu > 0$, the proof is complete.
\end{proof}

\begin{lemma}
    \label{lemma:1_2}
    Given any $M > 0, \quad | \tilde{\mathbf{\gamma}} |_1 < M \quad \Rightarrow \quad \dfrac{\partial \mathcal{L}_{1-2} + \mathcal{L}_{M}}{\partial \tilde{\mathbf{\gamma}}} z 0.$
\end{lemma}
\begin{proof}
    Similar to the previous lemma, we have that $\mathcal{L}_{M} = (1 + \mu) ~ (M - | \tilde{\mathbf{\gamma}} |_1)$ and $\mathcal{L}_{1-2} + \mathcal{L}_{M} = M - \| \tilde{\mathbf{\gamma}} \|^2_2 - \mu~\| \tilde{\mathbf{\gamma}} \|_1$. Its gradient is defined as $$\frac{\partial \mathcal{L}_{1-2} + \mathcal{L}_{M}}{\partial \tilde{\gamma_i}} = -2~ \tilde{\gamma_i} - \mu, \quad \forall i \in (1 \ldots F).$$ Having $\tilde{\gamma_i} \in [0, 1]$, we have that $$\max\left( \frac{\partial \mathcal{L}_{1-2} + \mathcal{L}_{M}}{\partial \tilde{\gamma_i}} \right) = -\mu, \quad \forall i \in (1 \ldots F).$$ Having $\mu > 0$, the proof is complete.
\end{proof}

\begin{theorem}
    The $\mathcal{L}_{1-2} + \mathcal{L}_{M}$ regularization term ensures that $| \tilde{\mathbf{\gamma}} |_1 \approx M$, preventing the features to fall from the extreme values (0 and 1) while maximizing the separation between them.
\end{theorem}
\begin{proof}
    By Lemma \ref{lemma:1_1} we can ensure that, if $| \tilde{\mathbf{\gamma}} |_1 > M$, all $\tilde{\mathbf{\gamma}}$ values will be reduced, preventing them to quickly reach the extreme value 1. Furthermore, by the same lemma, we know that the features which reduce their value the most are the ones that are closer to 0, forcing the system to drop the irrelevant features and maximizing their distance with respect to the relevant ones. On the contrary, if $| \tilde{\mathbf{\gamma}} |_1 < M$, lemma \ref{lemma:1_2} guarantees that all features will increase its value. And, again, the features which increase their value the most are the ones with higher $\tilde{\mathbf{\gamma}}$ values, maximizing their distance against the irrelevant ones.
\end{proof}

Below we provide an explanation about when Eq. \ref{eq:E2E-FS_soft} satisfies the restrictions imposed in our original problem.

\begin{lemma}
    \label{lemma:2_1}
    $\mathcal{L}_{1-2} = 0 \quad \Leftrightarrow \quad \tilde{\mathbf{\gamma}} \in \{0, 1\}$.
\end{lemma}
\begin{proof}
    Knowing that $\tilde{\mathbf{\gamma}} > 0$, we have that $\mathcal{L}_{1-2} = 0 ~\rightarrow~ \| \tilde{\mathbf{\gamma}} \|_1 - \| \tilde{\mathbf{\gamma}} \|^2_2 = 0 ~\rightarrow~ \sum_{i=1}^F \tilde{\gamma_i} - \tilde{\gamma_i}^2 = 0 ~\rightarrow~ \sum_{i=1}^F \tilde{\gamma_i}(1 - \tilde{\gamma_i}) = 0$. The rest is straightforward.
\end{proof}

\begin{lemma}
    \label{lemma:2_2}
    Given any $M > 0, \quad \mathcal{L}_{M} = 0 \quad \Leftrightarrow \quad \| \tilde{\mathbf{\gamma}} \|_1 = M$.
\end{lemma}
\begin{proof}
    By the definition of $\mathcal{L}_{M}$, it is straightforward.
\end{proof}

\begin{lemma}
    \label{lemma:2_3}
    Given any $M >0$ value, $\mathcal{L}_{1-2} + \mathcal{L}_{M} = 0 ~~ \Leftrightarrow ~~ \| \mathbf{\gamma} \|_1 = M$ and $\mathbf{\gamma} \in \{0, 1\}$.
\end{lemma}
\begin{proof}
    By Lemma \ref{lemma:2_1}, we know all $\mathbf{\gamma}$ values are integer. Thus, its 1-norm is also an integer. By Lemma \ref{lemma:2_2}, we know that $\| \mathbf{\gamma} \|_1 = M$, proving the lemma.
\end{proof}

\begin{theorem}
    Given any $M > 0$ value, a classifier $f$  with a \textit{E2E-FS} mask needs to be trained until $\mathcal{L}_{1-2} + \mathcal{L}_{M} = 0$ to ensure that $M$ features are selected and the binary mask is properly formed.
\end{theorem}
\begin{proof}
    By Lemma \ref{lemma:2_3}.
\end{proof}

\section{Experimental Results}
\label{sec:results}

\begin{table*}[t]
	\centering
	\caption{Feature selection approaches considered in the
experiments, specified according their time and memory complexity. N is the number of samples, F is the number of initial features, K is a multiplicative constant, i is the number of iterations in the case of iterative algorithms, C is the number of classes and $O_f$ is the classifier complexity (in wrappers).}
	\label{tab:methods}
	\resizebox{.6\linewidth}{!}{
    	\begin{tabular}{| c | c | c |} 
          	\hline
          	  \textbf{Method} & \textbf{Time complexity} & \textbf{Memory Complexity} \\
          	\hline
            MIM \cite{ross2014mutual} & $\mathcal{O}(N^2F^2)$ & $\mathcal{O}(F^2)$ \\
            \hline
            Fisher \cite{he2006laplacian} & $\approx \mathcal{O}(CNF)$ & $\mathcal{O}(F^2)$ \\
            \hline
            ReliefF \cite{kononenko1997overcoming} & $\mathcal{O}(iFNC)$ & $\mathcal{O}(F)$\\
            \hline
            InfFS \cite{roffo2015infinite} & $\mathcal{O}(N^{2.37}(1 + F))$ & $\mathcal{O}(F^2)$ \\
          	\hline
            ILFS \cite{roffo2017infinite} & $\mathcal{O}(N^{2.37} + iN + F + C)$ & $\mathcal{O}(F^2)$ \\
          	\hline
            DFS \cite{li2016deep} & $\mathcal{O}(O_f)$ & $\mathcal{O}(O_f + F)$\\
          	\hline
            SFS \cite{cancela2020scalable} & $\mathcal{O}(iO_f)$ & $\mathcal{O}(O_f + F)$\\
          	\hline \hline
            E2E-FS & $\mathcal{O}(O_f)$ & $\mathcal{O}(O_f + F)$\\
          	\hline 
        \end{tabular}
    }
\end{table*}

In order to test our algorithm\footnote{All our algorithms and scripts will be accessible via GitHub.}, we carried out a series of experiments over three different scenarios, varying the number of samples and features of the datasets: microarrays, datasets artificially modified for feature selection challenges \cite{guyon2005result, guyon2007agnostic}, and image datasets. Table \ref{tab:datasets} shows all datasets used. Our algorithm was implemented by using the Keras framework \cite{chollet2015keras}. All methods used to test against our methodology are implemented or accessible via Python scripts. Table \ref{tab:methods} summarizes all the methods used in this experimental section, along with both their time and memory complexity.
  
\begin{table*}
	\centering
	\caption{Datasets used in the experiments. The first block are microarray datasets (low number of samples and high number of features); the second are datasets specifically created to evaluate feature selection algorithms; the latter are image datasets, with a huge number of samples. An X in the last column means the dataset is balanced.}
	\label{tab:datasets}
	\resizebox{.8\linewidth}{!}{
    	\begin{tabular}{| c | c | c | c | c | c |} 
          	\hline
          	  \textbf{Dataset} & \textbf{\# samples} & \textbf{\# features} & \textbf{\# classes} & \textbf{unbalance (+/-)} \\
          	\hline
            LYMPHOMA \cite{golub1999molecular} & 45 & 4026 & 2 & 23/22\\
            COLON \cite{alon1999broad} & 62 & 2000 & 2 & 40/22\\
            LEUKEMIA \cite{golub1999molecular} & 72 & 7129 & 2 & 47/25\\
            LUNG \cite{gordon2002translation} & 181 & 12533 & 2 & 31/150\\
          	\hline\hline
            DEXTER \cite{guyon2007competitive} & 600 & 20000 & 2 & X\\
            MADELON \cite{guyon2007competitive} & 2600 & 500 & 2 & X\\
            GINA \cite{guyon2007agnostic} & 3153 & 970 & 2 & 1600/1500\\
            GISETTE \cite{guyon2007competitive} & 7000 & 5000 & 2 & X\\
          	\hline \hline
            MNIST \cite{lecun1998gradient}& 70K & 784 & 10 & X\\
            Fashion-MNIST \cite{xiao2017/online} & 70K & 784 & 10 & X\\
            CIFAR-10 \cite{krizhevsky2009learning} & 60K & 3072 & 10 & X\\
            CIFAR-100 \cite{krizhevsky2009learning} & 60K & 3072 & 100 & X\\
          	\hline 
        \end{tabular}
    }
\end{table*}

\vspace{0.3cm}\noindent\textbf{Microarrays}:\quad In order to train our algorithm we have created a naive SVM in Keras. It consists on a neural network with no hidden units and square-hinge as loss function, using weight balance. A $l2$-norm regularization is applied to the model's weights. It is set to $100/N$, being $N$ the number of samples in the training set. We have trained the model for $150$ epochs, using the Adam optimizer with a learning rate of $1e-3$, dividing its value by $5$ after every $50$ epochs. The batch size is set to $\max(2, N/50)$.

To guarantee the convergence of our methods, we added $300$ epochs to the beginning of the training (the learning rate will remain fixed to its initial value) for our E2E-FS algorithm, and $200$ for the E2E-FS-Soft (the moving parameter $M_{0.75}$ causes a faster convergence), setting the hyper-parameter $T = 300$ and $T = 250$ for E2E-FS and E2E-FS-Soft, respectively. We have a weight warm-up for $5$ epochs (only $\mathcal{L}_f$ is taken into account). All the other parameters will remain as default. As data normalization we have used the function
\begin{equation}
    \tilde{X} = erf\left(\frac{X - \mu_t}{2\sigma_t}\right),
\end{equation}
being $\mu_t$ and $\sigma_t$ the training set feature-wise mean and sample standard deviation. To test our algorithm, we followed the same procedure reported in \cite{roffo2015infinite}: we use a stratified 3-fold cross validation over the complete dataset, testing the algorithms against both our network previously described and a Linear SVM, which $C$ parameter is chosen by performing a grid search over the training dataset (a 5-fold partition is used for this matter). As some datasets have unbalanced data, we have used the area under the curve of the balance accuracy (AuC-BA) as our quality measure, averaging the performance obtained with the first 10, 50, 100, 150, and 200 selected features. As our method cannot be directly used over a Linear SVM model trained with the SMO algorithm, we also report the result of training the SVM with the features selected by our algorithm when trained with our naive network.

We decided to test our algorithms against five different feature selection approaches: Fisher \cite{he2006laplacian}, MIM \cite{ross2014mutual}, ReliefF \cite{kononenko1997overcoming}, InfFS \cite{roffo2015infinite}
and ILFS \cite{roffo2017infinite}. In order to have a fair competition between all FS methods, we have used a default configuration for all of them. To do so, we set InfFS and ILFS parameters without any cross-validation ($\alpha = 0.5$ for the InfFS and $\mathcal{T} = 6$ for the ILFS). Table \ref{tab:results} shows the obtained results. Our algorithms can achieve at least state-of-the-art results in all datasets. The results also show a tendency that will be confirmed later: the differences between our proposal and the other FS techniques increase with the number of samples. It is also to be noted the bad results obtained for both InfFS and ILFS. The first one was expected, as InfFS is the only unsupervised method we are testing. On the contrary, the ILFS method achieves very good results with 100 or more features, but its performance significantly drops below than point. This is caused because the algorithm assigns the highest score to a huge number of features (between 20 to 80, depending on the dataset), and subsequently it cannot distinguish between all of them.

% \usepackage{booktabs}
\begin{table*}
    \caption{Microarray and FS Challenge AUC-BA results, averaging the performance obtained with the first 10, 50, 100, 150, and 200 features. The same dataset splits were performed in every FS method (stratified 3-fold, 20 splits). NaiveF means our Naive network started with all features. Naive is the same network but with the features previously selected (In the case of our methods, they are selected by the NaiveF model). In bold face, best methods when using a pairwise Wilcoxon test with $\alpha = 0.05$. The datasets are ordered by the number of samples.}
    \label{tab:results}
    \resizebox{.99\linewidth}{!}{
        \begin{tabular}{@{}l|c||c|c|c|c||c|c|c|@{}}
            \cmidrule(l){3-9} 
            \multicolumn{2}{c|}{} & \textbf{LYMPHOMA} & \textbf{COLON} & \textbf{LEUKEMIA} & \textbf{LUNG} & \textbf{DEXTER} & \textbf{GINA} & \textbf{GISETTE} \\
            \midrule
            \multirow{2}{*}{\textbf{MIM}} & \textbf{Naive} & \textbf{0.944 $\pm$ 0.06} & 0.815 $\pm$ 0.08 & \textbf{0.970 $\pm$ 0.03} & \textbf{0.982 $\pm$ 0.02} & 0.862 $\pm$ 0.05 & 0.839 $\pm$ 0.03 & 0.919 $\pm$ 0.03 \\
            & \textbf{SVM} & 0.943 $\pm$ 0.06 & 0.807 $\pm$ 0.07 & 0.957 $\pm$ 0.03 & 0.983 $\pm$ 0.02 & 0.867 $\pm$ 0.04 & 0.834 $\pm$ 0.03 & 0.921 $\pm$ 0.03 \\
            \midrule
            \multirow{2}{*}{\textbf{FISHER}} & \textbf{Naive} & \textbf{0.948 $\pm$ 0.06} & 0.833 $\pm$ 0.06 & \textbf{0.969 $\pm$ 0.04} & 0.979 $\pm$ 0.03 & 0.904 $\pm$ 0.05 & 0.838 $\pm$ 0.03 & 0.918 $\pm$ 0.03\\
            & \textbf{SVM} & \textbf{0.950 $\pm$ 0.06} & 0.828 $\pm$ 0.07 & 0.959 $\pm$ 0.04 & \textbf{0.983 $\pm$ 0.02} & 0.909 $\pm$ 0.05 & 0.831 $\pm$ 0.03 & 0.919 $\pm$ 0.03\\
            \midrule
            \multirow{2}{*}{\textbf{RELIEFF}} & \textbf{Naive} & 0.942 $\pm$ 0.05 & 0.824 $\pm$ 0.08 & 0.961 $\pm$ 0.04 & 0.978 $\pm$ 0.03 & 0.908 $\pm$ 0.04 & 0.826 $\pm$ 0.04 & 0.916 $\pm$ 0.04\\
            & \textbf{SVM} & 0.944 $\pm$ 0.06 & 0.815 $\pm$ 0.07 & 0.957 $\pm$ 0.04 & 0.982 $\pm$ 0.02 & 0.910 $\pm$ 0.03 & 0.817 $\pm$ 0.03 & 0.921 $\pm$ 0.04\\
            \midrule
            \multirow{2}{*}{\textbf{InfFS}} & \textbf{Naive} & 0.852 $\pm$ 0.12 & 0.803 $\pm$ 0.09 & 0.943 $\pm$ 0.06 & 0.950 $\pm$ 0.09 & 0.812 $\pm$ 0.09 & 0.722 $\pm$ 0.06 & 0.844 $\pm$ 0.11\\
            & \textbf{SVM} & 0.837 $\pm$ 0.12 & 0.791 $\pm$ 0.08 & 0.925 $\pm$ 0.05 & 0.892 $\pm$ 0.02 & 0.810 $\pm$ 0.09 & 0.723 $\pm$ 0.06 & 0.843 $\pm$ 0.11\\
            \midrule
            \multirow{2}{*}{\textbf{ILFS}} & \textbf{Naive} & 0.868 $\pm$ 0.16 & 0.831 $\pm$ 0.07 & 0.952 $\pm$ 0.04 & 0.931 $\pm$ 0.12 & 0.805 $\pm$ 0.16 & 0.749 $\pm$ 0.13 & 0.888 $\pm$ 0.08\\
            & \textbf{SVM} & 0.866 $\pm$ 0.16 & 0.812 $\pm$ 0.07 & 0.941 $\pm$ 0.05 & 0.953 $\pm$ 0.10 & 0.801 $\pm$ 0.16 & 0.748 $\pm$ 0.13 & 0.886 $\pm$ 0.08\\
            \midrule
            \midrule
            \multirow{3}{*}{\textbf{E2E-FS}} & \textbf{NaiveF} & 0.929 $\pm$ 0.06 & \textbf{0.848 $\pm$ 0.06} & \textbf{0.971 $\pm$ 0.03} & 0.981 $\pm$ 0.03 & \textbf{0.923 $\pm$ 0.02} & \textbf{0.856 $\pm$ 0.02} & \textbf{0.961 $\pm$ 0.02} \\
            & \textbf{Naive} & 0.933 $\pm$ 0.06 & \textbf{0.847 $\pm$ 0.06} & 0.960 $\pm$ 0.04 & \textbf{0.985 $\pm$ 0.02} & \textbf{0.924 $\pm$ 0.02} & \textbf{0.854 $\pm$ 0.02} & \textbf{0.962 $\pm$ 0.02}\\
            & \textbf{SVM} & 0.929 $\pm$ 0.06 & 0.825 $\pm$ 0.07 & \textbf{0.972 $\pm$ 0.02} & 0.981 $\pm$ 0.03 & \textbf{0.923 $\pm$ 0.02} & \textbf{0.855 $\pm$ 0.02} & \textbf{0.961 $\pm$ 0.02}\\
            \midrule
            \multirow{3}{*}{\textbf{E2E-FS-Soft}} & \textbf{NaiveF} & 0.923 $\pm$ 0.07 & \textbf{0.846 $\pm$ 0.07} & \textbf{0.970 $\pm$ 0.03} & 0.981 $\pm$ 0.03 & 0.920 $\pm$ 0.02 & \textbf{0.855 $\pm$ 0.02} & \textbf{0.961 $\pm$ 0.03} \\
            & \textbf{Naive} & 0.935 $\pm$ 0.06 & \textbf{0.847 $\pm$ 0.06} & 0.959 $\pm$ 0.03 & \textbf{0.986 $\pm$ 0.02} & \textbf{0.922 $\pm$ 0.02} & \textbf{0.855 $\pm$ 0.02} & \textbf{0.963 $\pm$ 0.02} \\
            & \textbf{SVM} & 0.932 $\pm$ 0.06 & 0.820 $\pm$ 0.07 & \textbf{0.971 $\pm$ 0.03} & 0.981 $\pm$ 0.03 & 0.920 $\pm$ 0.02 & \textbf{0.855 $\pm$ 0.02} & \textbf{0.961 $\pm$ 0.02} \\
            \bottomrule
        \end{tabular}
    }
\end{table*}

\vspace{0.3cm}\noindent\textbf{FS Challenge datasets}:\quad Secondly, we tested our algorithm against challenging datasets that contain distractor features. Last three rows of Table \ref{tab:results} shows the obtained results when using the same configurations of the microarray datasets. As the number of samples is higher, our proposed models achieved the best results in all datasets, as observed above. Note that we did not introduce the MADELON dataset in this table, as this is a dataset that cannot be correctly classified by using a linear classifier. Furthermore, it is known that only 5 features are relevant, containing another 15 features that are linear combinations of the relevant features. Thus, it is useless to evaluate its performance when selecting more features. 

As the number of samples in these datasets are higher, we also performed a test using a dense network with $3$ hidden layers ($50$, $25$ and $10$ units, respectively). Batch Normalization \cite{ioffe2015batch} and the ReLU function were used in each layer. During the feature selection extra epochs, we increase the learning rate to $0.05$. All the other configurations remained as in the previous tests. This configuration was specifically chosen because it clearly overfits the training. We want to check how our algorithm behaves whenever the classifier is not ideal. Figure \ref{fig:nn_results} shows the results on the three datasets, when varying the number of features selected. Again, our algorithms (specially the E2E-FS-Soft) achieved the best results. It is specially remarkable the differences obtained when the number of features is extremely low. In the case of the MADELON dataset, it causes an interesting behavior. IAs mentioned above, this dataset contains only 5 relevant features along with 15 linear combinations of them. Our algorithms obtain their best results using only 5 features, suggesting that, whenever the classifier is not carefully chosen, the algorithm is prone to remove redundant information first, rather than noise features.

\begin{figure*}[t]
	\centering
	\begin{minipage}[]{0.32\linewidth}
	    \centering
 		\includegraphics[width=0.99\textwidth]{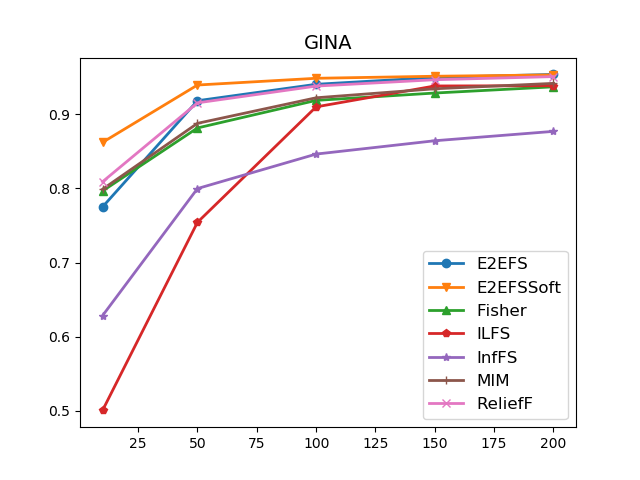}\\
	\end{minipage}
	\begin{minipage}[]{0.32\linewidth}
	    \centering
 		\includegraphics[width=0.99\textwidth]{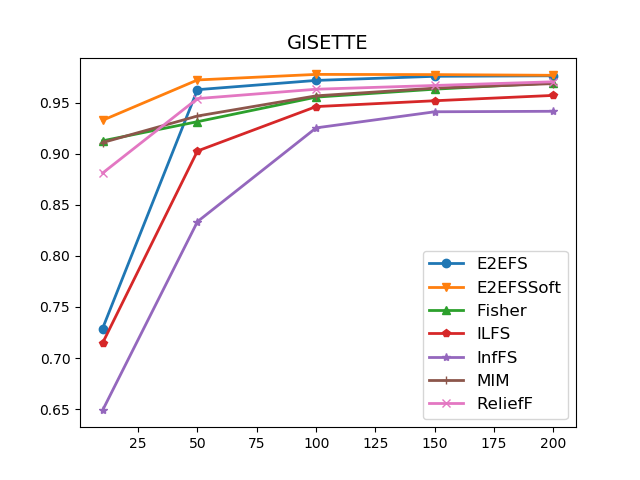}\\
	\end{minipage}
	\begin{minipage}[]{0.32\linewidth}
	    \centering
 		\includegraphics[width=0.99\textwidth]{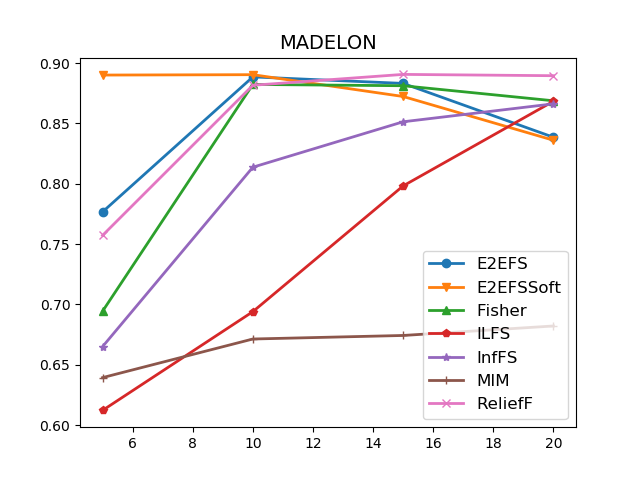}\\
	\end{minipage}
\caption{
 Balance Accuracy (BA) results (vertical axis) when using a three layer neural network. Our algorithms can maintain a high BA even when the number of features (horizontal axis) is low. 
}
\label{fig:nn_results}
\end{figure*}

\begin{table}
	\centering
	\caption{MNIST accuracy results when selecting a different amount of features (second row), using WRN-16-4 as the classifier. State-of-the-art results were obtained in  \cite{cancela2020scalable}. The T column refers to the time needed to perform the whole computation for the minimum number of features presented in the table.}
	\label{tab:mnist}
	%\resizebox{.85\linewidth}{!}{
	    \begin{tabular}{| c | c | c | c | c | c |} 
          	\toprule
          	  \multirow{2}{*}{} & \multicolumn{5}{|c|}{\textbf{MNIST}}\\
          	\cmidrule(l){2-6}
          	  & 39 & 78 & 196 & 392 & T(39) \\
          	\midrule
          	%DFS \cite{cancela2020scalable}  & 3-layer CNN & 95.73 & 98.56 & 99.32 & 99.46 \\
          	%\midrule
          	%DFS \cite{cancela2020scalable}  & WRN-16-4 & 92.92 & 97.36 & 98.72 & 98.98 \\
          	DFS & 95.73 & 98.56 & 99.32 & 99.46 & $3960s$\\
          	\midrule
          	%SFS \cite{cancela2020scalable} & 3-layer CNN ($\gamma = 0$) & 89.78 & 95.66 & 99.14 & 99.53 \\
          	%\midrule
          	%SFS \cite{cancela2020scalable} & 3-layer CNN ($\gamma = 0.9$) & \textbf{97.08} & 98.49 & 99.13 & \textbf{99.56} \\
          	%\midrule
          	%SFS \cite{cancela2020scalable} & WRN-16-4 ($\gamma = 0$) & 88.18 & 94.98 & 98.70 & 99.41 \\
          	%\midrule
          	%SFS \cite{cancela2020scalable} & WRN-16-4 ($\gamma = 0.9$) & 96.76 & 98.62 & 99.10 & 99.30 \\
          	SFS & 89.78 & 95.66 & 99.14 & \textbf{99.53} & $5940s$\\
          	\midrule
          	iSFS & \textbf{97.08} & \textbf{98.62} & 99.13 & \textbf{99.56} & $\approx 32h$\\
          	\midrule
          	%SFS$+$DFS \cite{cancela2020scalable} & 3-layer CNN ($\gamma = 0$) & 95.60 & 98.47 & \textbf{99.38} & 99.48 \\
          	%\midrule
          	%SFS$+$DFS \cite{cancela2020scalable} & WRN-16-4 ($\gamma = 0$) & 92.92 & 97.36 & 98.72 & 98.98 \\
          	SFS$+$DFS & 95.60 & 98.47 & \textbf{99.38} & 99.48& $5940s$ \\
          	\midrule
          	% \multirow{2}{*}{E2E-FS-Hard} & \textbf{97.07} & 98.64 & \textbf{99.37} & 99.42 \\
          	% & \textbf{$\pm$ 0.30} & $\pm$ 0.16 & \textbf{$\pm$ 0.05} & $\pm$ 0.05 \\
          	\midrule
          	\multirow{2}{*}{E2E-FS} & 91.87 & 97.54 & 99.28 & 99.49 & \multirow{2}{*}{$\mathbf{3420s}$} \\
          	& $\pm$ 3.81 & $\pm$ 0.9 & $\pm$ 0.08 & $\pm$ 0.04 & \\
          	\midrule
          	\multirow{2}{*}{E2E-FS-Soft} & 95.44 & 97.89 & 99.15 & 99.30 & \multirow{2}{*}{$\mathbf{3420s}$} \\
          	& $\pm$ 0.53 & $\pm$ 0.13 & $\pm$ 0.09 & $\pm$ 0.04 & \\
          	\bottomrule
        \end{tabular}
    %}
\end{table}

\vspace{0.2cm}\noindent\textbf{Image datasets}:\quad We also test how our algorithms deal with more complex classifiers like CNNs, by performing the same experiment presented in \cite{cancela2020scalable}: four datasets (MNIST, Fashion-MNIST, CIFAR-10 and CIFAR-100) and one network: the Wide Residual Network \cite{zagoruyko2016wide} (WRN-16-4), which was tested against four diferent approaches: the Deep Feature Selection (DFS) method \cite{li2016deep}, which is a variant of the LASSO algorithm, specially designed to be used in Deep Learning architectures; the SFS algorithm and its iterative version we called it iSFS; and a combination of both DFS and SFS. In this case, we followed the predefined dataset partitions. As our algorithms' hyper-parameters, we set $T=100$ and $T=250$ for E2E-FS and E2E-FS-Soft, respectively. First, we started the model weights with the result of training the model for $110$ epochs with all features. To do that, we used the SGD optimizer, starting with a learning rate of $0.1$, dividing its value by $5$ every $30$ epochs; we also used a $l2$-norm regularization over the weights, set at $5e-4$. After that, we started training our model for $70$ epochs to ensure the $M$ desired features are selected (learning rate fixed at $0.1$). Finally, we trained the model with the same specifications provided before. As we are dealing with CNNs, $\mathbf{\gamma}$ needs to be reshaped to the image input size in $\mathcal{L}_f$. 

Tables \ref{tab:mnist}, \ref{tab:fashionmnist}, \ref{tab:cifar10} and \ref{tab:cifar100} show the obtained results. The E2E-FS-Soft algorithm achieved the best results in all datasets but MNIST, in which the obtained results are lower than those that can be obtained by the DFS algorithm. We believe that our algorithm may not behave as well as expected against datasets with binary variables. However, as it has been said, the accuracy remains close to the state-of-the-art. On the contrary, the accuracy significantly rises in the other datasets. Again, a huge improvement is obtained when the number of selected features is low (with 153 variables, performance increases on CIFAR 10 and 100 by 12\% and 29\%, respectively).

\begin{table}
	\centering
	\caption{Fashion-MNIST accuracy results when selecting a different amount of features (second row), using WRN-16-4 as the classifier. State-of-the-art results were obtained in  \cite{cancela2020scalable}. The T column refers to the time needed to perform the whole computation for the minimum number of features presented in the table.}
	\label{tab:fashionmnist}
	%\resizebox{.85\linewidth}{!}{
    	\begin{tabular}{| c | c | c | c | c | c |} 
          	\toprule
          	\multirow{2}{*}{} &
          	 \multicolumn{5}{|c|}{\textbf{Fashion-MNIST}}\\
          	\cmidrule(l){2-6} 
          	% \midrule
          	%\textbf{MNIST}  & 
          	& 39 & 78 & 196 & 392 & T(39) \\
          	\midrule
          	%DFS \cite{cancela2020scalable} & 3-layer CNN & 78.85 & 85.5 & 90.45 & 92.41 \\
          	%\midrule
          	%DFS \cite{cancela2020scalable} & WRN-16-4 & 72.58 & 76.52 & 87.05 & 92.61 \\
          	DFS & 
          	78.85 & 85.50 & 90.45 & 92.61 & $4400s$ \\
          	\midrule
          	%SFS \cite{cancela2020scalable} & 3-layer CNN ($\gamma = 0$) & 67.85 & 81.86 & 89.33 & 92.36 \\
          	% \midrule
          	%SFS \cite{cancela2020scalable} & 3-layer CNN ($\gamma = 0.9$) & 82.63 & 86.33 & 90.09 & 92.60 \\
          	%\midrule
          	%SFS \cite{cancela2020scalable} & WRN-16-4 ($\gamma = 0$) & 64.17 & 73.53 & 85.60 & 92.48 \\
          	%\midrule
          	%SFS \cite{cancela2020scalable} & WRN-16-4 ($\gamma = 0.9$) & 77.99 & 82.58 & 88.16 & 91.72 \\
          	SFS & 
          	67.85 & 81.86 & 89.33 & 92.36 & $6600s$ \\
          	\midrule
          	iSFS & 
          	82.63 & 86.33 & 90.09 & 92.60 & $\approx 35h$ \\
          	\midrule
          	%SFS$+$DFS \cite{cancela2020scalable} & 3-layer CNN ($\gamma = 0$) & 79.81 & 86.29 & 90.44 & 92.59 \\
          	%\midrule
          	%SFS$+$DFS \cite{cancela2020scalable} & WRN-16-4 ($\gamma = 0$) & 65.09 & 73.87 & 85.86 & 92.40 \\
          	SFS$+$DFS & 
          	79.81 & 86.29 & 90.44 & 92.59 & $6600s$ \\
          	\midrule
          	%\multirow{2}{*}{E2E-FS-Hard} & 
          	% 84.50 & 89.37 & 92.00 & 92.44 \\
          	%&  
          	% $\pm$ 0.65 & $\pm$ 0.29 & $\pm$ 0.15 & $\pm$ 0.23 \\
          	\midrule
          	\multirow{2}{*}{E2E-FS} & 
          	78.11 & 83.06 & 91.45 & \textbf{93.88} & \multirow{2}{*}{$\mathbf{3800s}$}\\
          	& 
          	$\pm$ 3.34 & $\pm$ 1.30 & $\pm$ 0.50 & \textbf{$\pm$ 0.09} &  \\
          	\midrule
          	\multirow{2}{*}{E2E-FS-Soft} & 
          	\textbf{86.16} & \textbf{89.93} & \textbf{92.70} & 93.68 & \multirow{2}{*}{$\mathbf{3800s}$}\\
          	& 
          	\textbf{$\pm$ 0.44} & \textbf{$\pm$ 0.14} & \textbf{$\pm$ 0.18} & $\pm$ 0.09 &  \\
          	\bottomrule
        \end{tabular}
    %}
\end{table}

\begin{table}
	\centering
	\caption{CIFAR-10 accuracy results when selecting a different amount of features (second row), using WRN-16-4 as the classifier. State-of-the-art results were obtained in  \cite{cancela2020scalable}. The T column refers to the time needed to perform the whole computation for the minimum number of features presented in the table.}
	\label{tab:cifar10}
	%\resizebox{.85\linewidth}{!}{
    	\begin{tabular}{| c | c | c | c | c | c |} 
          	\toprule
          	\multirow{2}{*}{} &
          	\multicolumn{5}{|c|}{\textbf{CIFAR-10}}\\
          	\cmidrule(l){2-6} 
          	%\midrule
          	& 153 & 307 & 768 & 1536 & T(153) \\
          	\midrule
          	%DFS \cite{cancela2020scalable} & 3-layer CNN & 78.85 & 85.5 & 90.45 & 92.41 \\
          	%\midrule
          	%DFS \cite{cancela2020scalable} & WRN-16-4 & 72.58 & 76.52 & 87.05 & 92.61 \\
          	DFS & 
          	67.43 & 79.92 & 87.71 & 90.69 & $4840s$ \\
          	\midrule
          	%SFS \cite{cancela2020scalable} & 3-layer CNN ($\gamma = 0$) & 67.85 & 81.86 & 89.33 & 92.36 \\
          	% \midrule
          	%SFS \cite{cancela2020scalable} & 3-layer CNN ($\gamma = 0.9$) & 82.63 & 86.33 & 90.09 & 92.60 \\
          	%\midrule
          	%SFS \cite{cancela2020scalable} & WRN-16-4 ($\gamma = 0$) & 64.17 & 73.53 & 85.60 & 92.48 \\
          	%\midrule
          	%SFS \cite{cancela2020scalable} & WRN-16-4 ($\gamma = 0.9$) & 77.99 & 82.58 & 88.16 & 91.72 \\
          	SFS & 
          	61.00 & 72.49 & 85.55 & 90.44 & $7260s$ \\
          	\midrule
          	iSFS & 
          	64.15 & 79.27 & 89.85 & 91.58 & $\approx 39h$ \\
          	\midrule
          	%SFS$+$DFS \cite{cancela2020scalable} & 3-layer CNN ($\gamma = 0$) & 79.81 & 86.29 & 90.44 & 92.59 \\
          	%\midrule
          	%SFS$+$DFS \cite{cancela2020scalable} & WRN-16-4 ($\gamma = 0$) & 65.09 & 73.87 & 85.86 & 92.40 \\
          	SFS$+$DFS & 
          	68.13 & 79.03 & 88.04 & 91.06 & $7260s$ \\
          	\midrule
          	%\multirow{2}{*}{E2E-FS-Hard} & 
          	%74.32 & 83.93 & 88.41 & 91.97 \\
          	%& 
          	%$\pm$ 0.28 & $\pm$ 0.16 & $\pm$ 0.72 & $\pm$ 0.08 \\
          	\midrule
          	\multirow{2}{*}{E2E-FS} & 
          	71.80 & 82.76 & 90.37 & \textbf{93.01} & \multirow{2}{*}{$\mathbf{4180s}$} \\
          	& $\pm$ 0.32 & $\pm$ 0.62 & $\pm$ 0.51 & \textbf{$\pm$ 0.18} & \\
          	\midrule
          	\multirow{2}{*}{E2E-FS-Soft} & 
          	\textbf{76.50} & \textbf{84.43} & \textbf{91.32} & 92.72 & \multirow{2}{*}{$\mathbf{4180s}$} \\
          	& \textbf{$\pm$ 0.87} & \textbf{$\pm$ 0.26} & \textbf{$\pm$ 0.15} & $\pm$ 0.17 & \\
          	\bottomrule
        \end{tabular}
    %}
\end{table}
\begin{table}
	\centering
	\caption{CIFAR-100 accuracy results when selecting a different amount of features (second row), using WRN-16-4 as the classifier. State-of-the-art results were obtained in  \cite{cancela2020scalable}. The T column refers to the time needed to perform the whole computation for the minimum number of features presented in the table.}
	\label{tab:cifar100}
	%\resizebox{.85\linewidth}{!}{
    	\begin{tabular}{| c | c | c | c | c | c |} 
          	\toprule
          	%\multirow{2}{*}{\textbf{CIFAR-100}}  &
          	\multirow{2}{*}{} &
          	\multicolumn{5}{|c|}{\textbf{CIFAR-100}}\\
          	\cmidrule(l){2-6}  
          	% \midrule
          	& 153 & 307 & 768 & 1536 & T(153) \\
          	\midrule
          	%DFS \cite{cancela2020scalable} & 3-layer CNN & 78.85 & 85.5 & 90.45 & 92.41 \\
          	%\midrule
          	%DFS \cite{cancela2020scalable} & WRN-16-4 & 72.58 & 76.52 & 87.05 & 92.61 \\
          	DFS & 
          	34.55 & 49.68 & 57.92 & 67.42 & $4840s$ \\
          	\midrule
          	%SFS \cite{cancela2020scalable} & 3-layer CNN ($\gamma = 0$) & 67.85 & 81.86 & 89.33 & 92.36 \\
          	% \midrule
          	%SFS \cite{cancela2020scalable} & 3-layer CNN ($\gamma = 0.9$) & 82.63 & 86.33 & 90.09 & 92.60 \\
          	%\midrule
          	%SFS \cite{cancela2020scalable} & WRN-16-4 ($\gamma = 0$) & 64.17 & 73.53 & 85.60 & 92.48 \\
          	%\midrule
          	%SFS \cite{cancela2020scalable} & WRN-16-4 ($\gamma = 0.9$) & 77.99 & 82.58 & 88.16 & 91.72 \\
          	SFS & 
          	24.66 & 37.86 & 56.66 & 66.39 & $7260s$ \\
          	\midrule
          	iSFS & 
          	30.74 & 44.55 & 62.83 & 67.22 & $\approx 39h$ \\
          	\midrule
          	%SFS$+$DFS \cite{cancela2020scalable} & 3-layer CNN ($\gamma = 0$) & 79.81 & 86.29 & 90.44 & 92.59 \\
          	%\midrule
          	%SFS$+$DFS \cite{cancela2020scalable} & WRN-16-4 ($\gamma = 0$) & 65.09 & 73.87 & 85.86 & 92.40 \\
          	SFS$+$DFS & 
          	36.86 & 46.64 & 60.46 & 64.94 & $7260$ \\
          	\midrule
          	%\multirow{2}{*}{E2E-FS-Hard} & 
          	% 45.10 & 55.55 & 64.26 & 68.39 \\
          	%& 
          	% $\pm$ 0.57 & $\pm$ 0.57 & $\pm$ 0.23 & $\pm$ 0.10 \\
          	\midrule
          	\multirow{2}{*}{E2E-FS} & 
          	40.48 & 53.32 & 63.08 & 69.78 & \multirow{2}{*}{$\mathbf{4180s}$} \\
          	& $\pm$ 1.00 & $\pm$ 1.15 & $\pm$ 0.53 & $\pm$ 0.15 & \\
          	\midrule
          	\multirow{2}{*}{E2E-FS-Soft} & 
          	\textbf{48.65} & \textbf{56.98} & \textbf{67.03} & \textbf{70.30} & \multirow{2}{*}{$\mathbf{4180s}$} \\
          	& \textbf{$\pm$ 0.45} & \textbf{$\pm$ 0.25} & \textbf{$\pm$ 0.18} & \textbf{$\pm$ 0.13} & \\
          	\bottomrule
        \end{tabular}
    %}
\end{table}

\section{Conclusion}
\label{sec:conclusion}

In this paper we have presented a novel feature selection method that can be attached to any model that is trained by using gradient descent techniques. We provided the idea of our algorithm as a general optimization problem, and an approximation that can be used to efficiently solve it. The experimental results show how our proposal, using only default parameters, can achieve state-of-the-art results or better in a wide variety of datasets (from microarray to larger image datasets), using different classifiers (SVMs, dense NNs and CNNs) and optimizers (Adam and SGD).

As future work, we believe that this idea can open several paths: 1) Develop different approaches that can solve the E2E-FS optimization problem; 2) Evaluate the possibility of using E2E-FS as a classic regularization technique, rather than force it to select a specific number of features. This approach will involve the combination of our optimization techniques with algorithms for selecting the optimal number of features \cite{somol2004fast, nguyen2010optimal}; and 3) Perform an extensive analysis to check if the algorithm have some limitations when it is used against binary data. Although the results are good, they seem to somehow limit the advantages of our algorithm against other FS techniques for all scenarios.

% \section*{Broader impact}

% Given that we are ubiquitously exposed to AI systems in all areas of our lives, there is an important need that these AI systems could be trustworthy. The definition of Explainable Artificial Intelligence (XAI) is not yet consensual, although authors intersect in aspects such as: explainable models, enabling human users to understand appropriately and trust. Thanks to this new XAI, we could implement a right to explanation, which is one of the pillars of the AI HLEG’s Ethics Guidelines\footnote{\url{https://ec.europa.eu/digital-single-market/en/news/ethics-guidelines-trustworthy-ai}} and closely connected with the principles and rights of the recent European GDPR law\footnote{\url{https://gdpr-info.eu}}. Among the different techniques that have been used  for achieving explainability, one of them is influence methods, which include feature selection methods. Feature selection is crucial for a correct interpretability of the data, since it indicates which features are relevant for a given task, and which features are irrelevant and/or redundant. The method presented in this work, by identifying the relevant features used by the classifiers, can identify if discrimination is made in decision-making (e.g. gender discrimination). In fact, the development of feature selection methods as the one presented here, can contribute to two of the Sustainable Development Goals adopted by the United Nations\footnote{\url{https://www.un.org/sustainabledevelopment/}}, in particular to SDG 5 (gender equality) and SDG 10 (reduced inequalities).

{\small
\printbibliography
}

\end{document}